%% file: main.tex
\documentclass[letterpaper]{article} 
\usepackage{aaai25}  
\usepackage{times}  
\usepackage{helvet}  
\usepackage{courier}  
\usepackage[hyphens]{url}  
\usepackage{graphicx} 
\usepackage{natbib}  
\usepackage{caption} 
\usepackage{algorithm}
\usepackage[noend]{algorithmic}
\usepackage{subcaption}
\usepackage{xcolor}
\usepackage{newfloat}
\usepackage{booktabs}
\usepackage{multirow}
\usepackage{listings}
\DeclareCaptionStyle{ruled}{labelfont=normalfont,labelsep=colon,strut=off} 
\floatstyle{ruled}
\newfloat{listing}{tb}{lst}{}
\floatname{listing}{Listing}
\usepackage{amsmath}
\usepackage{amssymb}
\usepackage{mathtools}
\usepackage{amsthm}

\theoremstyle{plain}
\newtheorem{theorem}{Theorem}[section]

\theoremstyle{definition}

\newcommand{\method}[0]{DSBD}

\newcommand{\nop}[1]{}

\title{Dynamic-Width Speculative Beam Decoding for LLM Inference}
\author{
    Zongyue Qin, 
    Zifan He,
    Neha Prakriya, 
    Jason Cong,
    Yizhou Sun
}
\affiliations{
    University of California Los Angeles, CA, USA\\
    \{qinzongyue, zifanhe, nehaprakriya, cong, yzsun\}@cs.ucla.edu
}

\usepackage{bibentry}

\begin{document}

\maketitle

\begin{abstract}
Large language models (LLMs) based on transformer architecture have shown outstanding performance across numerous real-world tasks. However, the autoregressive nature of these models makes the inference process slow and costly. 
Speculative decoding has emerged as a promising solution, leveraging a smaller auxiliary model to draft future tokens, which are then validated simultaneously by the larger model, achieving a speed-up of $1$-$2\times$. Although speculative decoding matches the same distribution as multinomial sampling, multinomial sampling itself is prone to suboptimal outputs, whereas beam sampling is widely recognized for producing higher-quality results by maintaining multiple candidate sequences at each step.
%
%
This paper explores the novel integration of speculative decoding with beam sampling. However, there are four key challenges: (1) how to generate multiple sequences from the larger model's distribution given draft sequences from the small model; (2) how to dynamically optimize the number of beams to balance efficiency and accuracy; (3) how to efficiently verify the multiple drafts in parallel; and (4) how to address the extra memory costs inherent in beam sampling.
To address these challenges, we propose dynamic-width speculative beam decoding (\method{}). Specifically, we first introduce a novel draft and verification scheme that generates multiple sequences following the large model's distribution based on beam sampling trajectories from the small model. Then, we introduce an adaptive mechanism to dynamically tune the number of beams based on the context, optimizing efficiency and effectiveness. Besides, we extend tree-based parallel verification to handle multiple trees simultaneously, accelerating the verification process. Finally, we illustrate a simple modification to our algorithm to mitigate the memory overhead of beam sampling.
Experimental results show that our approach achieves a $1.5$-$1.9\times$ speed-up and $1.8$-$2.5\times$ lower energy consumption compared to beam sampling, with no loss in downstream performance. Moreover, it can produce significantly higher-quality outputs than speculative decoding, while maintaining similar time, memory, and energy costs. In summary, our method offers a more efficient and effective inference process for LLMs.
\end{abstract}

%

\input{sections/intro}

\input{sections/prelim}

\input{sections/method}

\input{sections/experiment}

\input{sections/related_work}

\input{sections/conclusion}

\bibliography{aaai25}

\appendix
\clearpage
\input{sections/appendix}

\end{document}

%% file: sections/intro.tex
\section{Introduction}

In recent years, large language models based on transformer architecture~\cite{vaswani2017attention}, such as GPT-4~\cite{achiam2023gpt}, Llama-3~\cite{llama3modelcard}, and PALM~\cite{anil2023palm}, have demonstrated remarkable performance across a wide range of real-world tasks, including text generation, summarization, and translation. However, the autoregressive nature of these models, where tokens are generated one at a time, leads to slow inference speeds and high computational costs. As the size and complexity of LLMs continue to increase, the demands on computational resources and energy consumption during inference have become major concerns, limiting their scalability and accessibility.


Speculative decoding has emerged as a promising technique to accelerate LLM inference by leveraging a smaller auxiliary model to generate draft tokens. These tokens are then validated by the large model, resulting in a significant reduction in inference time. The primary advantage of speculative decoding is its ability to maintain the same quality of output as multinomial sampling while achieving a 1-2$\times$ speed-up. However, multinomial sampling itself is limited to generating a single sequence based on local optimality. This limitation makes it susceptible to returning suboptimal results, as it lacks the diversity that could be achieved by considering multiple candidate sequences simultaneously.

Motivated by the need to improve the output quality, 
we explore the integration of speculative decoding with beam sampling, a technique that maintains multiple candidate sequences (beams) at each step to enhance the diversity and quality of the generated output. This fusion, however, presents several challenges. First, while previous studies focused on obtaining a single token from the large model distribution given draft tokens from the smaller model, our approach requires generating multiple tokens (beams) simultaneously, which necessitates a new verification scheme. Second, determining the optimal number of beams is critical: too many beams can lead to inefficiency due to a high rejection rate, while too few beams may result in under-utilization of the small model's potential and low effectiveness. Third, efficiently verifying multiple draft sequences in parallel requires a technique that can process and validate multiple beams concurrently. Fourth, addressing the additional memory cost of storing multiple key-value caches is crucial to enable LLMs to use beam sampling in practice.

To address these challenges, we propose dynamic-width speculative beam decoding (\method{}) that combines speculative decoding with beam sampling through a series of innovations. First, we introduce a draft and verification scheme that processes beam decoding trajectories as forests of trees, which are verified layer by layer by the large model. This approach allows us to efficiently generate multiple beams while maintaining the large model's sampling distribution. Second, we propose a mechanism to dynamically adjust the number of beams based on the context, ensuring a balance between efficiency and effectiveness. Third, we extend existing tree-based parallel verification techniques~\cite{miao2023specinfer} to operate on multiple trees, incorporating a forest-based parallel verification strategy that enhances the speed of the verification process. Finally, we introduce a simple modification to \method{} that reduces the memory cost by storing only one set of key-value caches, while still delivering better output quality than multinomial sampling. 

Experimental results show that our approach achieves a $1.5$-$1.9\times$ speed-up and $1.8$-$2.5\times$ smaller energy consumption than beam sampling, without sacrificing performance on downstream tasks. Besides, it can produce significantly higher-quality outputs than speculative decoding, while maintaining comparable time, memory, and energy costs. These findings suggest that \method{} successfully bridges the gap between speculative decoding and beam sampling, providing a more efficient and effective decoding method for LLMs. Our code is open source\footnote{\url{https://github.com/ZongyueQin/DSBD}}.

%% file: sections/prelim.tex
\section{Preliminaries}

\begin{figure*}[ht]
    \centering
\begin{subfigure}{0.48\linewidth}
    \includegraphics[width=0.95\linewidth]{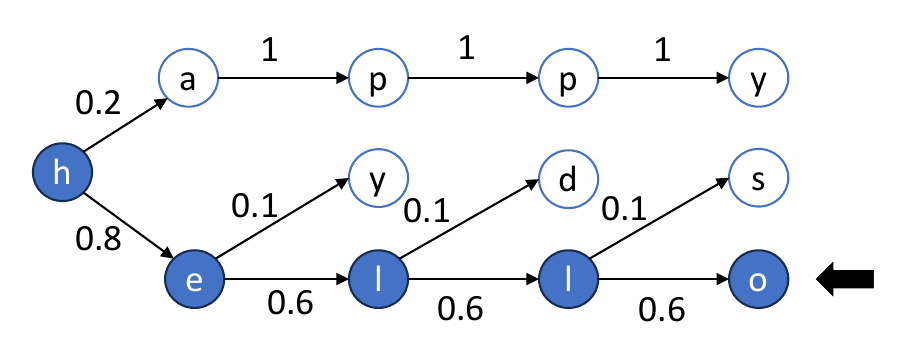}
    \caption{Multinomial sampling.}
\end{subfigure}
\begin{subfigure}{0.48\linewidth}
    \includegraphics[width=0.95\linewidth]{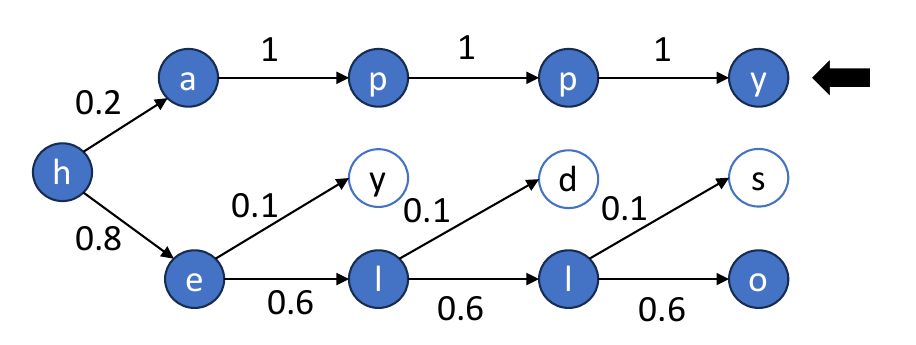}
    \caption{Beam sampling with 2 beams. 
    }
\end{subfigure}
     \caption{Examples of greedy and beam sampling. Some nodes are omitted in the figures. Assume the sampling probability is warped to always sample the tokens with the largest probabilities. Given $prefix$ ``\texttt{h}'', multinomial sampling generates ``\texttt{hello}'' with an average perplexity of 1.55. Beam sampling generates ``\texttt{happy}'' with an average perplexity of 1.49. 
     }
    \label{fig:greedy-beam-sampling}
\end{figure*}

\subsection{Decodings of LLMs}

Let $p$ denote the distribution defined by a large language model $M_p$. Given an input prefix, the optimal decoding algorithm is to generate a sequence of $N$ tokens with maximum likelihood $p(x_{1:N}|input)$. 

\paragraph{Multinomial Sampling.} Multinomial sampling, also known as standardized sampling, samples the next token $x_t$ based on $\mathcal{T}\circ p(\cdot|x_{1:t-1},input)$, where $\mathcal{T}$ is a warping operation applied to enhance the high probability region. Some common warping operations include \textit{top-k} warping, which limits the selection to the top $k$ tokens, and \textit{top-p} warping, where tokens are sampled from the smallest possible subset of the vocabulary whose cumulative probability mass exceeds a specified threshold $p$. 
The deterministic version of multinomial sampling is a special case when $k=1$. 

\paragraph{Beam Sampling.} 
Beam decoding aims to do a better job than multinomial sampling. 
For each position $t$ ($1\le t\le N$), it maintains $W>1$ candidate sequences, which are also called \emph{beams}. 
Assume we have already kept the $W$ sequences $\mathcal{I}_{t-1}=\{x_{1:t-1}^{(1)},\ldots,x_{1:t-1}^{(W)}\}$ at position $t-1$, $W$ sequences with length $t$ are then sampled from $\mathcal{T}\circ p_{beam}$, where $p_{beam}$:$\mathcal{I}_{t-1}\times V\rightarrow [0,1]$ is the beam sampling probability:
\begin{equation}
    p_{beam}(x_{1:t-1}^{(i)},x_t)=\frac{p(x_{1:t-1}^{(i)},x_t|input)}{\sum_{x_{1:t-1}^{(j)},x^\prime_t\in \mathcal{I}_{t-1}\times V }p(x_{1:t-1}^{(j)},x_t^\prime|input)}
\end{equation}
Notice that $p(x_{1:t-1}^{(i)},x_t|input)=p(x_t|x_{1:t-1}^{(i)},input)\cdot p(x_{1:t-1}^{(i)}|input)$. In practice, beam sampling stores the likelihood $p(x_{1:t-1}^{(i)}|input)$ for each beam, and the computation complexity of $p_{beam}$ is $O(W\cdot|V|)$. In deterministic beam sampling, the top $W$ sequences with the highest likelihood $p_{beam}(x_{1:t})$ will be kept. 

\cite{shi2024thorough} shows that beam sampling in general has better downstream effectiveness than multinomial sampling. Figure \ref{fig:greedy-beam-sampling} shows an example where beam decoding returns a better output.

\subsection{Vanilla Speculative Decoding~\label{sec:reject_sampling}}

Speculative decoding utilizes a small model to generate the next $\gamma$ tokens and then employs the large model to verify these drafted tokens \emph{in parallel}. The process is summarized as follows:

\begin{enumerate}
    \item Given $input$, the small model samples $\gamma$ draft tokens $x_1, \ldots, x_\gamma$ using greedy decoding, based on the warped predicted conditional probability $\tilde{q}(x_t|x_{1:t-1
},input)$ for $t=1, \ldots, \gamma$, where $\tilde{q} = \mathcal{T} \circ q$ and $q$ is the small model's output distribution.
\item The large model verifies the draft tokens in parallel by computing the conditional probability $\tilde{p}=\mathcal{T}\circ p(x_t|x_{1:t-1
},input)$ for $t=1, \ldots, \gamma$.
\item Each draft token $x_t$ is accepted with probability $\min(1, \tilde{p}(x_t)/\tilde{q}(x_t))$. The draft tokens before the first rejected token are kept as the decoding output. An additional token is sampled from a residual distribution as a correction for the first rejected token. The accepted tokens and the resampled token are then appended to the context $prefix$ as input for the next iteration.

\item Repeat steps 1-3 until reaching the stopping criteria, such as a length limit.

\end{enumerate}

By verifying $\gamma$ tokens in parallel with one run of the large model, speculative decoding reduces the time cost compared to calling the large model $\gamma$ times. Additionally, although the small model still runs in an autoregressive manner, its inference speed is much faster than the large model. This makes speculative decoding an effective method to accelerate the inference process of LLMs. Moreover, it has been proven that each token $x_t$ generated by speculative sampling follows the identical sampling distribution as multinomial sampling.

%% file: sections/method.tex
\section{Methodology}

\begin{figure*}
    \centering
    \begin{subfigure}{0.32\textwidth}
        \centering
        \includegraphics[width=0.99\linewidth]{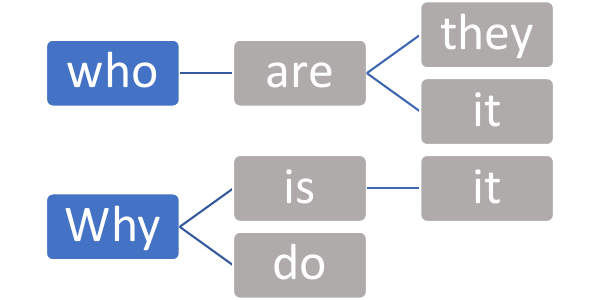}
        \caption{Draft forest from the small model.}
        \label{fig:SBD_L0}
    \end{subfigure}
    \begin{subfigure}{0.32\textwidth}
        \centering
        \includegraphics[width=0.99\linewidth]{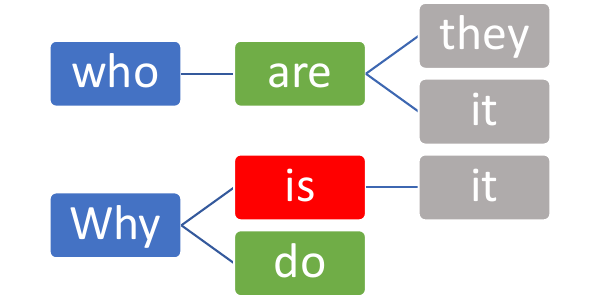}
        \caption{Verification result of the 1st draft layer.}
        \label{fig:SBD_L1}
    \end{subfigure}
    \begin{subfigure}{0.32\textwidth}
        \centering
        \includegraphics[width=0.99\linewidth]{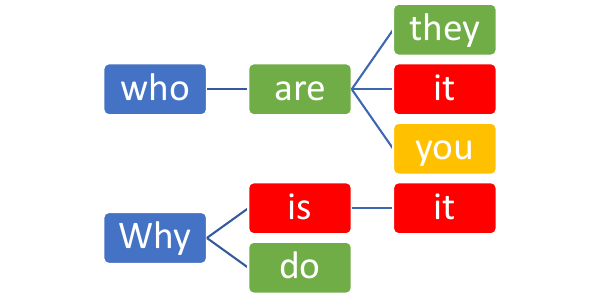}
        \caption{Verification result of the 2nd draft layer.}
        \label{fig:SBD_L2}
    \end{subfigure}
    \caption{Illustration of one iteration of Speculative Beam Decoding. (a) Draft Stage: given the input beams ``\textit{who}'' and ``\textit{why}'', the small model first generates a trace of beam sampling. (b)(c): Verification Stage. When verify the first draft layer, ``\textit{who are}'' and ``\textit{why do}'' are accepted, while ``\textit{why is}'' is rejected. When verify the second draft layer, ``\textit{why is it}'' is directly rejected because its parent is rejected. Then ``\textit{who are they}'' is accepted, while ``\textit{who are it}'' is rejected. And another beam ``\textit{who are you}'' is sampled from the residual distribution.}
    \label{fig:SBD}
\end{figure*}

The primary goal of our method is to enhance the efficiency and effectiveness of large language model (LLM) inference by combining the speed advantages of speculative decoding with the accuracy and diversity benefits of beam sampling. 
We first introduce a novel draft and verification scheme that keeps identical distribution as beam sampling. Then, we describe an adaptive beam management strategy. Next, we illustrate a forest-based parallel verification mechanism. Finally, we discuss how to resolve the additional memory cost inherent in beam sampling.

\subsection{Draft and Verification Scheme}

\subsubsection{Overview}
As illustrated in Figure \ref{fig:SBD}, the core idea of our method is to leverage a smaller, auxiliary model to generate multiple draft sequences, referred to as draft beams, which are then verified and refined by the larger model. This approach enables us to maintain multiple candidate sequences throughout the decoding process, thereby achieving better output quality than multinomial sampling, while improving the overall efficiency of beam sampling.

For now, assume that the number of beams (also referred to as the width, denoted as $W_L$) is fixed. In each iteration of our method, the input consists of the beams generated in the previous iteration. For the first iteration, the input is the initial input context. 
At each iteration, our method first uses the small model to perform beam sampling with a width of $W_S$ for $\gamma$ steps. Notice that we want $W_S>W_L$ because some draft beams might be rejected later. As illustrated in Figure \ref{fig:SBD_L0}, it generates a trajectory that can be represented as a forest consisting of $W_L$ trees, which we refer to as the ``draft forest''. In this forest, each tree originates from an input beam, with the maximum depth of each tree being $\gamma + 1$. Starting from the second layer, each layer of the forest contains $W_S$ nodes, representing the intermediate beams at each step of the beam sampling process.

Once the draft forest is generated, our method leverages the large model to predict the distribution for the next token of each node (beam) in parallel. Using these distributions, \method{} then verifies each layer of the draft forest sequentially. For each layer, it calculates the joint probability of the beams and sequentially determines whether each beam should be accepted. If $W_L$ beams are accepted in a given layer, the remaining beams are discarded, and the method is moved on to verify the next layer. If fewer than $W_L$ beams are accepted in layer $l$, the method rejects this layer and terminates the verification process.

When verification ends, either because it is terminated or because there are no more layers to verify, our method samples an additional layer with $W_L$ beams. This additional layer either corrects the first rejected layer or adds a new layer if all draft layers are accepted. The output beams from this additional layer then serve as the input beams for the next iteration, continuing until the stopping criteria are met (e.g., reaching the maximum number of tokens).

This approach allows each run of the large model to produce at least one, and possibly multiple, steps of beam sampling. Previous studies have shown that memory operations during LLM runs contribute significantly to both runtime and energy consumption~\cite{leviathan2023fast, allen2016characterizing, chen2011tree}. By generating multiple tokens in a single run, \method{} reduces the number of memory operations required, which in turn improves both the speed and the energy efficiency of LLM inference.


\subsubsection{Details}
Let $p$ denote the output distribution of the large model and $q$ denote the distribution of the small model. 
We will start by explaining how to verify the first draft layer (which is the second layer of the draft forest) during each iteration. 

Let $\mathcal{I}=\{x_{1:t}^{(1)},\cdots,x_{1:t}^{(W_L)}\}$ represent the input beams, and $\mathcal{S}=\{x_{1:t+1}^{(1)},\cdots,x_{1:t+1}^{(W_S)}\}$ represent the draft beams in the first layer of the draft forest. Note that $x_{1:t+1}^{(i)}$ is sampled from the distribution $q_{beam}(x_{1:t+1}^{(i)})= \mathcal{T}\circ\frac{q({x_{1:t+1}^{(i)}})}{\mathcal{Q}}$, where $\mathcal{T}$ denotes the warping operation and $\mathcal{Q}=\sum_{x_{1:t+1}\in \mathcal{I}\times V}q(x_{1:t+1})$. 
Similarly, let $p_{beam}$ denote the beam sampling distribution of the large model, we have $p_{beam}(x_{1:t+1}^{(i)})= \mathcal{T}\circ\frac{p({x_{1:t+1}^{(i)}})}{\mathcal{P}}$, where $\mathcal{P}=\sum_{x_{1:t+1}\in \mathcal{I}\times V}p(x_{1:t+1})$. 

During verification, our method starts by setting $p^\prime=p_{beam}$. For each draft beam $x_{1:t}^{(i)}$, \method{} accepts it with probability $\min(1,\frac{q_{beam}(x_{1:t}^{(i)})}{p^{\prime}(x_{1:t}^{(i)})})$. If $x_{1:t}^{(i)}$ is rejected, the method updates $p^\prime$ by setting it to $norm(\max(0, p^{\prime}-q_{beam}))$, where $norm$ denotes the normalization operation. Then it continues to verify the next draft beam with the updated $p^\prime$. If the beam is accepted, $p^{\prime}$ is reset to $p_{beam}$. If $W_L$ draft beams have already been accepted in this layer, the method will reject all remaining beams.

Now we illustrate how to verify the second draft layer. The difference is that some beams in the first layer have already been rejected. In this case, all the beams stem from the rejected beams are directly rejected. For the remaining beams, \method{} applies the same verification process as above.

If all layers in the draft forest have $W_L$ accepted beams, the method proceeds to sample an additional layer with $W_L$ beams directly from $p_{beam}$. However, if at any layer $l$ fewer than $W_L$ beams are accepted, the method will first sample one beam from the adjusted distribution $p^{\prime}$. If the number of accepted beams still falls short of $W_L$, additional beams will be sampled from the original distribution $p_{beam}$ to meet the required number.

\begin{theorem}
    \textbf{Correctness of Draft and Verification Scheme}.
    Let $\mathcal{I}=\{x_{1:t}^{(1)},\cdots,x_{1:t}^{(W_L)}\}$ denote input beams, $\mathcal{S}=\{x_{1:t+1}^{(1)},\cdots,x_{1:t+1}^{(W_S)}\}$ denote draft beams, and $\mathcal{O}=\{\tilde{x}_{1:t+1}^{(1)},\cdots,\tilde{x}_{1:t+1}^{(W_L)}\}$ denote the output beams obtained by our algorithm. We have $\tilde{x}_{1:t+1}^{(i)}\overset{\mathrm{iid}}{\sim} p_{beam}$ ($\forall i=1,\ldots,W_L$), where $p_{beam}(x_{1:t+1}^{(i)})= \mathcal{T}\circ(p({x_{1:t+1}^{(i)}})/\mathcal{P})$, $\mathcal{P}=\sum_{x_{1:t+1}\in \mathcal{I}\times V}p(x_{1:t+1})$. 
\end{theorem}
The proof is illustrated in~\cite{qin2024dynamic}.

\begin{algorithm}[t!]
\caption{Draft and Verification for Speculative Beam Sampling}
\label{alg:draft_verify}
\begin{algorithmic}[1]
\STATE \textbf{Input:} Draft Forest with $\gamma$ draft layers, Small model distribution $q$, Large model distribution $p$, Beam width $W_L$, $W_S$.
\STATE \textbf{Output:} Verified beams for the next iteration
\STATE $l_{last}\leftarrow \gamma+1$ 
\FOR{$l=1,\ldots,\gamma$}
    \STATE \COMMENT{$\mathcal{I}^{(l)}$ is the beams of layer $l-1$ in the forest.}
    \STATE $\mathcal{I}^{(l)}\leftarrow$input beams of layer $l$.  
    \STATE \COMMENT{$\mathcal{S}^{(l)}$ is the beams of layer $l$ in the forest.}
    \STATE $\mathcal{S}^{(l)}\leftarrow$draft beams of layer $l$.
    \STATE \COMMENT{remove beams stem from beams rejected in the last layer}
    \STATE $\mathcal{S}^{(l)}\leftarrow \{x_{1:t+1}^{(l,i)}|x_{1:t+1}^{(l,i)}\in \mathcal{S}^{(l)}, x_{1:t}^{(l,i)} \text{is not rejected}\}$
    \STATE \COMMENT{$t+1$ is the length of sequence in $\mathcal{S}^{(l)}$, $t=l-1$.}
    \STATE compute $p_{beam}^{(l)}$ based on next-token distributions $p$
    \STATE $p^\prime\leftarrow p_{beam}^{(l)}$
    \STATE Compute $W_L^{(l)}$ based on Eq \ref{eq:prob1} - Eq. \ref{eq:WL}
    \FOR{$x_{1:t+1}^{(l,i)}\in \mathcal{S}^{(l)}$}
        \STATE $r\leftarrow U(0,1)$
        \IF{$r\le \frac{q_{beam}^{(l)}(x_{1:t+1}^{(i)})}{p^\prime(x_{1:t+1}^{(i)})}$}
            \STATE accept $x_{1:t+1}^{(l,i)}$
            \STATE $p^\prime\leftarrow p_{beam}^{(l)}$
        \ELSE
            \STATE reject $x_{1:t+1}^{(l,i)}$
            \STATE $p^\prime\leftarrow \text{norm}(\max(0, p^\prime - q_{beam}^{(l)}))$
        \ENDIF
        \IF{$W_L^{(l)}$ beams are accepted}
            \STATE reject remaining beams
            \STATE \textbf{break}
        \ENDIF
    \ENDFOR
    \IF{less than $W_L^{(l)}$ beams are accepted}
        \STATE sample $x_{1:t+1}\sim p^\prime$ and add it to accepted beams
        \WHILE {less than $W_L^{(l)}$ beams are accepted}
            \STATE sample $x_{1:t+1}\sim p_{beam}^{(l)}$ and add it to accepted beams
        \ENDWHILE
        \STATE $l_{last}\leftarrow l$ 
        \STATE \textbf{break}
    \ENDIF
\ENDFOR
\IF {$l_{last} == \gamma+1$}
    \STATE compute $p_{beam}^{(\gamma+1)}$
    \STATE sample $W_L$ beams from $p_{beam}^{(\gamma+1)}$ 
\ENDIF
\STATE \RETURN accepted beams at the layer $l_{last}$
\end{algorithmic}
\end{algorithm}


\subsection{Dynamic-Width Speculative Beam Decoding}

The draft and verification scheme described above ensures that our method matches the sampling distribution of beam sampling. However, it has a limitation: the beam width $W_L$ remains fixed across all layers. While this fixed width works well for standard beam sampling, it is not suitable for our method. The key challenge is that the discrepancy between the small model's predictions ($q_{beam}$) and the large model's true distribution ($p_{beam}$) vary from token to token. In some layers, $q_{beam}$ closely aligns with $p_{beam}$, resulting in a high acceptance rate. In other layers, the gap is much wider, leading to a lower acceptance rate. 

To address this variability, the decoding algorithm should dynamically adjust the number of beams it expects to accept based on the alignment between $q_{beam}$ and $p_{beam}$. By doing so, it can (1) reduce the target width for challenging layers, preventing the entire layer from being rejected and thus maintaining progress, and (2) increase the target width for easier layers, enhancing the exploration of diverse sequences and reducing the risk of getting trapped in local optima. This adaptive approach would optimize the balance between efficiency and accuracy, making the decoding process more robust and effective.
So we propose a self-adjusting method where the target width $W_L^{(l)}$ for layer $l$ is determined based on the context of that layer.

Let $P_{p,q}^{(l)}(m,k)$ represent the probability that $k$ out of $m$ draft beams are accepted at the $l$-th layer. This probability is computed using the following recursive equation:
\begin{equation}
    P_{p,q}^{(l)}(m,k)=\sum_{i=1}^{m}\tilde{P}^{(l)}_{p,q}(m,i)P_{p,q}(m-i,k-1))
    \label{eq:prob1}
\end{equation}
Here, $\tilde{P}^{(l)}_{p,q}(m,i)$ is the probability that the $i$-th beam is the first to be accepted among the $m$ draft beams:
\begin{equation}
    \tilde{P}^{(l)}_{p,q}(m,i)=\alpha_i^{(l)}\prod_{j=1}^{i-1}(1-\alpha_j^{(l)})
    \label{eq:prob2}
\end{equation}
where $\alpha_j^{(l)}$ is the probability that the $j$-th beam is accepted, given that all previous beams (from the 1st to the $(j-1)$-th) were rejected.
\begin{equation}
\alpha_j^{(l)}=\sum q_{beam}\min(p_j^{(l)}/q_{beam}^{(l)},1)
\end{equation}
where $p_1^{(l)}=p_{beam}^{(l)}$, $p_k^{(l)}=norm(\max(p_{k-1}^{(l)}-q^{(l)}_{beam},0))$.    

Using these equations and the fact that $P^{(l)}_{p,q}(m,k) = 0$ if $k > m$ and $P^{(l)}_{p,q}(0,0)=1$, we can calculate the probability that at least $K$ beams are accepted at the $l$-th layer as: 
\begin{equation}
 1-\sum_{k=1}^{K-1}P_{p,q}^{(l)}(M_S,k)
 \label{eq:intro_WL}
\end{equation}
Finally, the width $W_L^{(l)}$ for the $l$-th layer is set based on Eq \ref{eq:intro_WL}, ensuring that it is not less than a minimum width $W_{min}$:
\begin{equation}
 W_L^{(l)}=\max(W_{min}, \tilde{W}_L^{(l)}(t))   
 \label{eq:WL}
\end{equation}
In this expression, $t \in [0,1]$ is a pre-defined threshold. The value of $\tilde{W}_L^{(l)}(t)$ is computed as follows:
\begin{equation}
    \tilde{W}_L^{(l)}(t) =\max\{K\in\mathbb{N}|1-\sum_{k=0}^{K-1}P_{p,q}^{(l)}(M_S,k)\ge t\}
\end{equation}
This formula gives us the maximum width $\tilde{W}_L^{(l)}(t)$ such that the probability of accepting at least $\tilde{W}_L^{(l)}(t)$ beams at the $l$-th layer is greater than or equal to the threshold $t$. Eq \ref{eq:WL} ensures that the width is dynamically adjusted to maintain a high likelihood of accepting a sufficient number of beams, while also ensuring that it does not fall below the minimum width $W{min}$. Algorithm \ref{alg:draft_verify} illustrates the pseudocode for the draft and verification scheme.

Let $\beta^{(l)}_{W_{min}}=\sum_{k=W_{min}}^{W_S}P^{(l)}_{p,q}(W_S,k))$, which is the probability that at least $W_{min}$ beams are accepted at layer $l$. Based on the definition of $W_L^{(l)}$, the probability that layer $L$ is accepted is $\min(t, \beta^{(l)}_{W_{min}})$. So $t$ and $W_{min}$ both control the average acceptance rate of our algorithm, and hence determine efficiency. Let $\bar{\beta}=\mathbb{E}\beta^{(l)}_{W_{min}}$, we have the following theorem for the efficiency of \method{}.

\begin{theorem}
    The expected number of steps generated per iteration is $\frac{1-\min(t,\beta)^{\gamma+1}}{1-\min(t,\beta)}$.
\end{theorem}
\begin{proof}
    As described above, the average acceptance rate is $\min(t,\beta)$. With the Theorems in~\cite{leviathan2023fast}, we can calculate the average number of generated layers as $\frac{1-\min(t,\beta)^{\gamma+1}}{1-\min(t,\beta)}$.
\end{proof}


\subsection{Forest-based Parallel Decoding
}

As noted in \cite{miao2023specinfer}, efficient management of the key-value cache is crucial to avoid redundant computations when running the large model during verification, which affects overall efficiency. SpecInfer \cite{miao2023specinfer} introduced tree-based parallel decoding, which reuses the same key-value cache and employs a topology-aware causal mask to accelerate the computation of the large model. However, this tree-based parallel decoding approach cannot be directly applied to our algorithm because, unlike SpecInfer, our method retains multiple beams as inputs at each iteration. Although these beams share the same initial input, the tokens generated in each beam can differ significantly as the sequence length increases. As a result, the draft tokens in \method{} form not a single tree but a forest.

So we propose forest-based parallel decoding, an extension of tree-based parallel decoding that accommodates multiple trees. As shown in Figure \ref{fig:forest_attention}, \method{} maintains a separate key-value cache for each input beam. For each beam, we apply tree-based parallel decoding to compute the tree attention across all tokens. Finally, after the iteration ends, \method{} updates the key-value caches according to the output beams. For example, if the output beams in Figure \ref{fig:forest_attention} are $b_5$ and $b_6$, which both originate from $b_1$, then the caches for $b_5$ and $b_6$ are kept for the next iteration.

\subsection{Reducing Memory Cost\label{sec:memory}}

In practice, key-value caches take up a large portion of memory cost for LLM inference~\cite{kang2024gear}. A critical disadvantage of beam sampling is that it has to maintain a separate key-value cache for each beam, significantly increasing the memory cost. But our method can mitigate this issue with a simple modification.
Notice that with the forest-based parallel decoding, the number of key-value caches kept during generation equals the number of input beams. So an effective way to reduce the memory cost of our method is to limit the number of input beams. This can be achieved by selecting only the output beam with the lowest perplexity as the input beam for the next iteration. In this way, only one key-value cache is needed during generation, so the memory cost will be similar to the cost of multinomial sampling and speculative decoding. Notice that although there is only one input beam, more than one beam can be accepted at each layer of the draft forest. Hence, it will be more effective than multinomial sampling.

\begin{figure}
    \centering
    \includegraphics[width=0.9\linewidth]{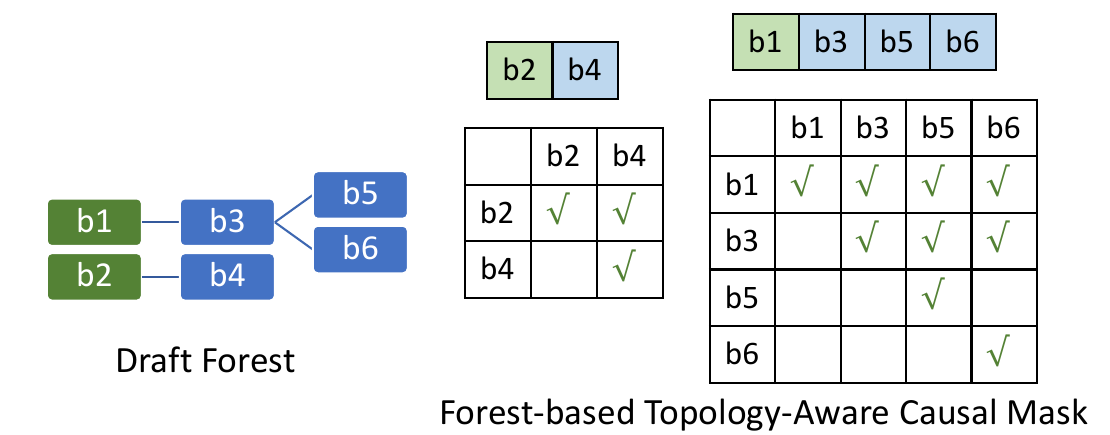}
    \caption{Illustration of forest-based parallel decoding. Given the draft forest, the large model converts the two trees into sequences in depth-first search order and verifies them in parallel with the topology-aware attention mask. Empty cells in the matrices indicate that attention is masked.}
    \label{fig:forest_attention}
\end{figure}

%% file: sections/experiment.tex
\section{Experiments}

\subsection{Experiment Setups}

\textbf{LLMs}. We evaluate our method using three publicly available LLM families: OPT~\cite{zhang2022opt}, Llama-2 and Llama-3~\cite{touvron2023llama,llama3modelcard}. We use Llama-2-13B, Llama-3.1-8B, and OPT-13B as the large models as they are the largest models our GPU could run. And we use Llama-68M~\cite{miao2023specinfer}
, Llama-3.2-1B, and OPT-125M as the small models. 

\textbf{Datasets}. We use public datasets: SQuAD~\cite{squad}, Spider~\cite{yu2018spider}, and MT-Bench~\cite{zheng2023judging}. SQuAD is a natural language QA dataset using exact match (EM) as the evaluation metric. Spider is a text-to-SQL code dataset that uses execution accuracy (EA) as the metric. MT-bench covers various tasks including writing, roleplay, extraction, stem, humanities, reasoning, math, and coding. It uses GPT-4 to rate the output quality on a scale of 1-10 (the higher the better).\footnote{Additional experiments and reproduction details are available in our arXiv version~\cite{qin2024dynamic}.}

\subsection{Comparison with Beam Sampling}

\begin{figure}[h!]
    \centering
    \begin{subfigure}[b]{0.23\textwidth}
        \centering
        \includegraphics[width=\textwidth]{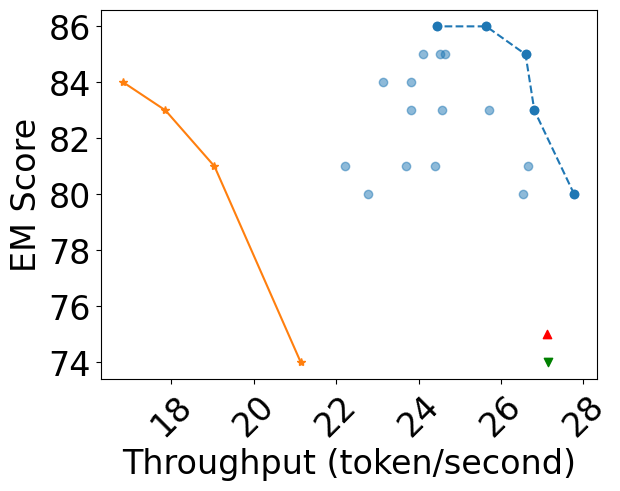}
        \captionsetup{justification=centering}
        \caption{Throughput vs EM Score\\ (Llama-2)}
        \label{fig:llama_squad_speed}
    \end{subfigure}
    \hfill
    \begin{subfigure}[b]{0.23\textwidth}
        \centering
        \includegraphics[width=\textwidth]{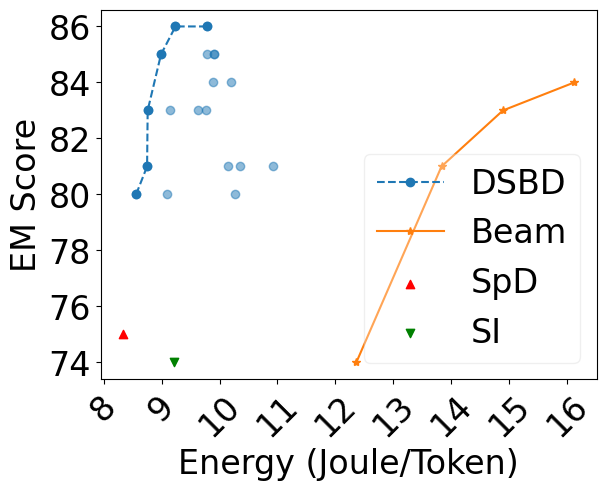}
        \captionsetup{justification=centering}
        \caption{Energy vs EM Score\\ (Llama-2)}
        \label{fig:llama_squad_energy}
    \end{subfigure}
    \begin{subfigure}[b]{0,23\textwidth}
        \centering
        \includegraphics[width=\textwidth]{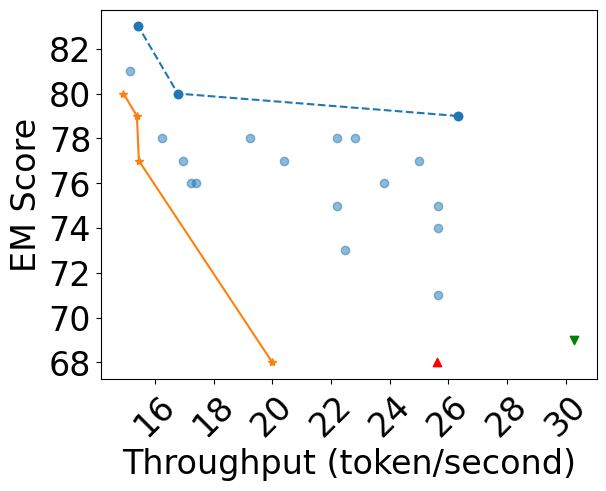}
        \captionsetup{justification=centering}
        \caption{Throughput vs EM Score\\ (Llama-3)}
        \label{fig:L3_squad_speed}
    \end{subfigure}
    \hfill
    \begin{subfigure}[b]{0,23\textwidth}
        \centering
        \includegraphics[width=\textwidth]{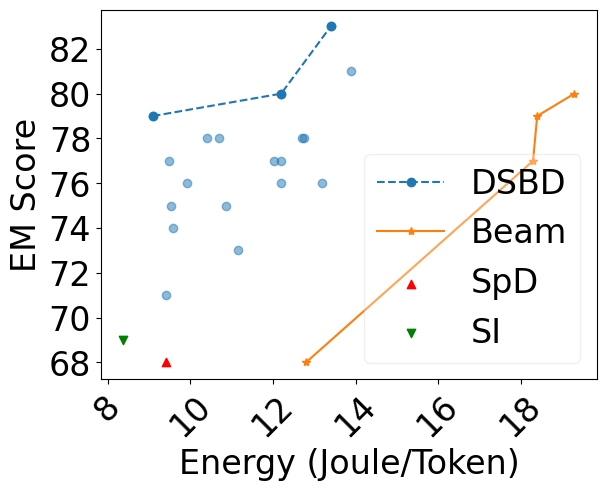}
        \captionsetup{justification=centering}
        \caption{Energy vs EM Score\\ (Llama-3)}
        \label{fig:L3_squad_energy}
    \end{subfigure}
        \begin{subfigure}[b]{0,23\textwidth}
        \centering
        \includegraphics[width=\textwidth]{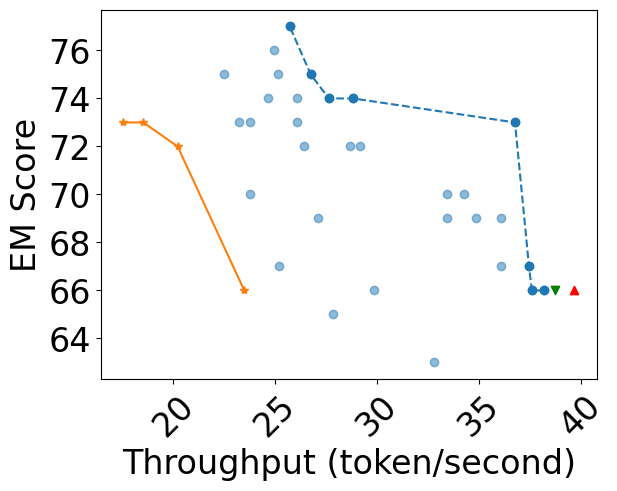}
        \captionsetup{justification=centering}
        \caption{Throughput vs EM Score\\ (OPT)}
        \label{fig:opt_squad_speed}
    \end{subfigure}
    \hfill
    \begin{subfigure}[b]{0,23\textwidth}
        \centering
        \includegraphics[width=\textwidth]{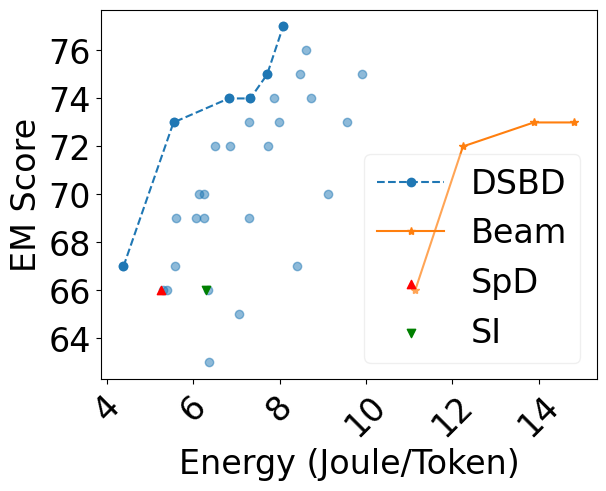}
        \captionsetup{justification=centering}
        \caption{Energy vs EM Score\\ (OPT)}
        \label{fig:opt_squad_energy}
    \end{subfigure}
    \caption{Evaluation on SQuAD. Exact match (EM) is \textbf{higher} the better. The blue points represent performances of \method{} under different parameter settings $(\gamma,W_S,t)$. The blue and yellow lines mark the Pareto fronts of \method{} and beam sampling. (SpD: SpecDecode, SI: SpecInfer)}
    \label{fig:squad}
    \vspace{-6mm}
\end{figure}

We begin by comparing \method{} with beam sampling, focusing on the relationship between efficiency (e.g., energy consumption and throughput) and effectiveness. The width of beam sampling ranges from 1 to 4. When width equals 1, beam sampling is equivalent to multinomial sampling. In addition, we observe the improvement in downstream effectiveness and output perplexity begins to converge when the width reaches around 4. For our method, we vary the draft beam width $W_S\in\{2,3,4,5,6\}$, the threshold $t\in\{0.7,0.9\}$, and set $W_{min}\in\{1,2,3\}$. We also include speculative decoding~\cite{leviathan2023fast} (SpD) and SpecInfer~\cite{miao2023specinfer} (SI) as references.

Figure \ref{fig:squad} and Figure \ref{fig:spider} illustrate the points that mark the performance of different methods under different parameter settings on SQUAD and Spider datasets, respectively. SpD and SpecInfer each have only one point in the figures because they do not offer a trade-off between efficiency and effectiveness. We plot the curves of beam sampling and the Pareto fronts of \method{}. 
Notably, we omit the results of the OPT model on the Spider dataset as its execution accuracy remains consistently close to zero, rendering it uninformative for this analysis. 
The figures demonstrate that \method{} consistently outperforms beam sampling, signifying that it achieves higher quality at any given level of throughput or energy consumption. More importantly, when the effectiveness is fixed, \method{} can be $1.5$-$1.9\times$ faster than beam sampling, while reducing energy consumption by $1.8$-$2.5\times$, as demonstrated by the Pareto fronts of \method{}. 
Table \ref{tab:MT-bench-beam} presents the results on MT-Bench. Due to the cost and time of GPT-4 evaluations, we report results for SpecInfer, beam sampling ($W=5$), and \method{}. \method{} achieves comparable efficiency to SpecInfer while significantly improving output quality. It is also $1.53\times$ faster and $1.54\times$ more energy-efficient than beam sampling. These results highlight \method{}'s advantages in efficiency and effectiveness, making it ideal for real-world applications.

\begin{figure}[h!]
    \centering
    \begin{subfigure}[b]{0,23\textwidth}
        \centering
        \includegraphics[width=\textwidth]{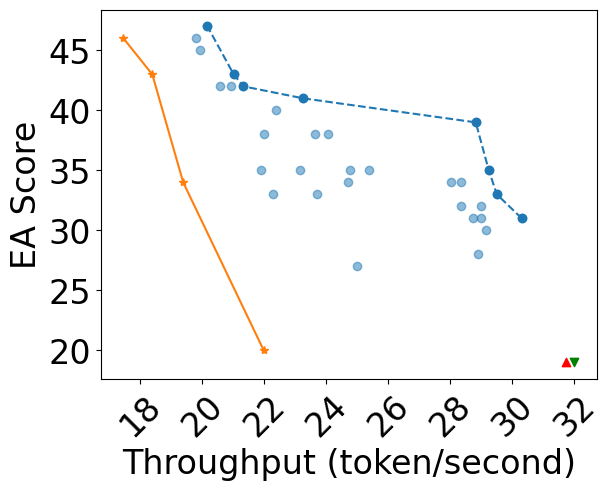}
        \captionsetup{justification=centering}
        \caption{Speed vs EA Score\\(Llama-2)}
        \label{fig:llama_spider_speed}
    \end{subfigure}
    \hfill
    \begin{subfigure}[b]{0,23\textwidth}
        \centering
        \includegraphics[width=\textwidth]{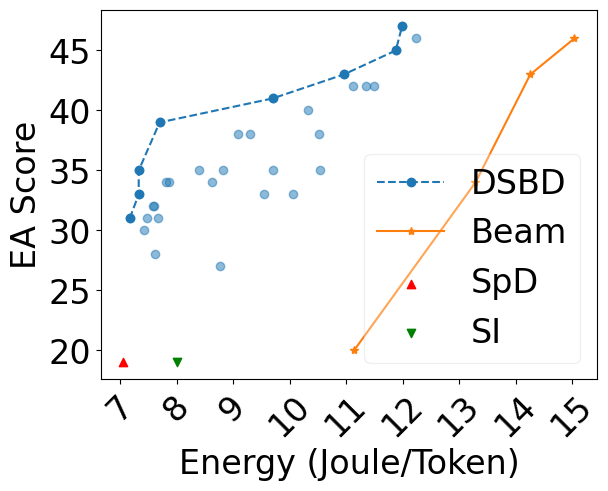}
        \captionsetup{justification=centering}
        \caption{Energy vs EA Score\\(Llama-2)}
        \label{fig:llama_spider_energy}
    \end{subfigure}
       \begin{subfigure}[b]{0,23\textwidth}
        \centering
        \includegraphics[width=\textwidth]{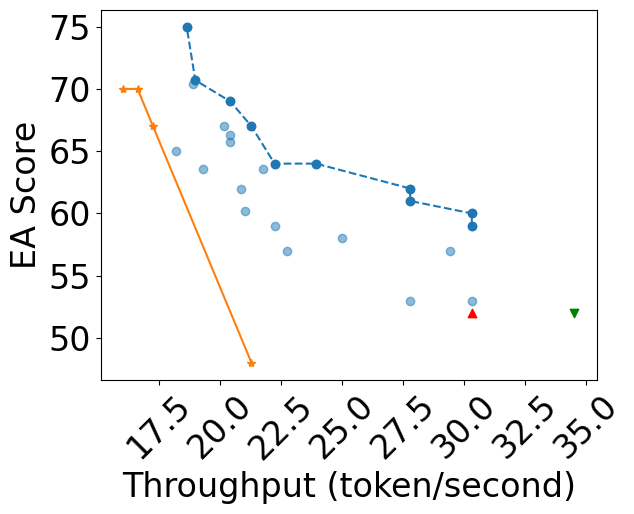}
        \captionsetup{justification=centering}
        \caption{Speed vs EA Score\\(Llama-3)}
        \label{fig:L3_spider_speed}
    \end{subfigure}
    \hfill
    \begin{subfigure}[b]{0,23\textwidth}
        \centering
        \includegraphics[width=\textwidth]{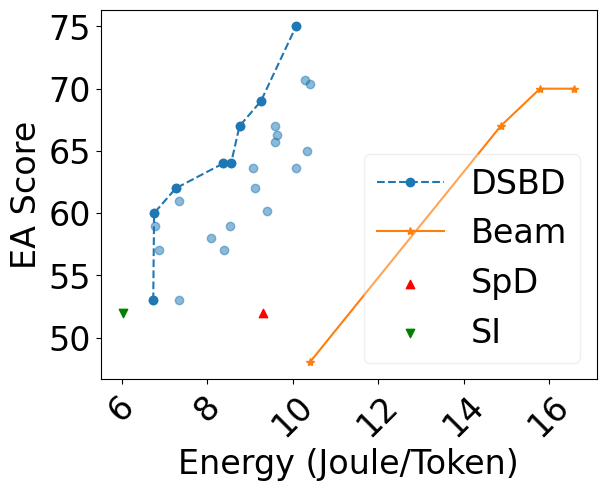}
        \captionsetup{justification=centering}
        \caption{Energy vs EA Score\\(Llama-3)}
        \label{fig:L3_spider_energy}
    \end{subfigure}
    \caption{Evaluation on Spider. Execution accuracy (EA) is \textbf{higher} the better. The blue points represent performances of \method{} under different parameter settings $(\gamma,W_S,t)$. The blue and yellow lines mark the Pareto fronts of \method{} and beam sampling.  (SpD: SpecDecode, SI: SpecInfer)}
    \label{fig:spider}
\end{figure}

\begin{table}[ht]
\centering
\footnotesize
\begin{tabular}{ccccc}
\toprule
\textbf{Model}       & \textbf{Method} & \textbf{Score} & \textbf{Token/s} & \textbf{J/token} \\ \midrule
 & SpecInfer & 2.86 & \textbf{21.8} & \underline{21.2}\\
Llama-2-13B          & Beam ($W$=5)        & \underline{3.51}           & 12.1             & 26.3            \\ 
                     & DSBD            & \textbf{3.52}           & \underline{16.5 }            & \textbf{16.1}            \\ \hline
        & SpecInfer & 3.46 & \textbf{20.2} & \textbf{19.8}\\
Llama-3-8B           & Beam ($W$=5)        & \underline{4.10}           & 10.5             & 33.3            \\ 
                     & DSBD            & \textbf{4.11}           & \underline{17.8}             & \underline{22.9}            \\ 
\bottomrule
\end{tabular}

\caption{Evaluation on MT-Bench with SpecInfer, beam sampling and \method{}.}
\label{tab:MT-bench-beam}
\vspace{-3mm}
\end{table}

\nop{
\begin{figure}[h!]
    \centering
    \begin{subfigure}[b]{0,23\textwidth}
        \centering
        
        \includegraphics[width=\textwidth]{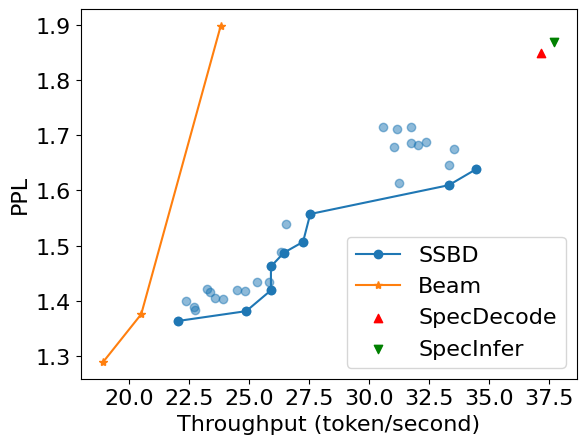}
        \caption{Speed vs Perplexity}
        \label{fig:opt_spider_speed}
    \end{subfigure}
    \hfill
    \begin{subfigure}[b]{0,23\textwidth}
        \centering
        \includegraphics[width=\textwidth]{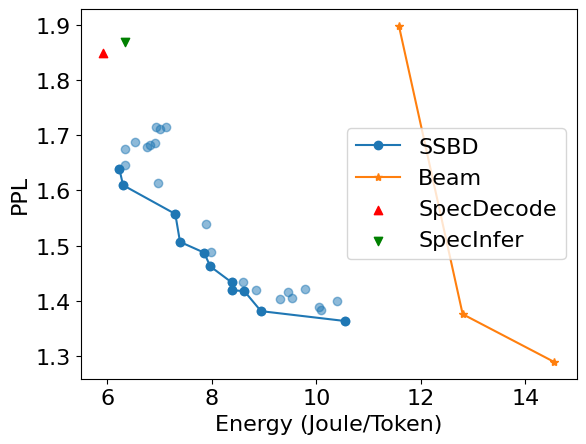}
        \caption{Energy vs Perplexity}
        \label{fig:opt_spider_energy}
    \end{subfigure}
    
    \caption{Evaluation on Spider dataset with OPT models. Perplexity (PPL) is \textbf{lower} the better. The blue points represent performances of \method{} under different parameter settings $(\gamma,W_S,t)$. The blue and yellow lines mark the Pareto fronts of \method{} and beam sampling.}
    \label{fig:opt_spider}
\end{figure}
}

\subsection{Comparison under Memory Constraint}

As discussed in Section \ref{sec:memory}, \method{} can mitigate the memory issue of beam sampling by selecting only one output beam for the next iteration. It allows \method{} to only keep one sequence of key-value cache and to achieve memory usage comparable to that of multinomial sampling.
To assess the performance of \method{} under memory constraints (i.e., only keeps one sequence of key-value cache), we compare it with multinomial sampling, SpD, and SpecInfer, as shown in Table \ref{tab:memory_constraint}. In addition, the \method{} in Table \ref{tab:MT-bench-beam} also only keeps one sequence of key-value cache. The results show that \method{} achieves speed and energy efficiency close to that of SpD. Moreover, \method{} delivers a significant improvement in output quality, far surpassing the baselines in downstream scores. 

\begin{table}[h]
\centering
\footnotesize
\begin{tabular}{lcccc}
\toprule
 & Method & EM/EA & tokens/s & J/token \\
\midrule
 &Multinomial & 74 &  21.14&  12.36\\
      Llama-2 & SpD & 75& 27.11 & \textbf{8.34}\\
      SQuAD & SpecInfer & 74  & \textbf{27.15} & 9.22 \\
                        & \method{}      & \textbf{86} &  26.67 & 8.75 \\
\midrule


 &Multinomial & 20 &  21.98&  11.14\\
      Llama-2 & SpD & 19&  31.74 & \textbf{7.06}\\
      Spider & SpecInfer & 19  & \textbf{32.00} & 8.01 \\
                        & \method{}      & \textbf{31} & 30.30 & 7.17 \\


\bottomrule
\end{tabular}
\caption{Comparison under memory constraints: each method stores key-value caches for only one sequence. 
}
\label{tab:memory_constraint}
\vspace{-3mm}
\end{table}

%% file: sections/related_work.tex
\section{Related Work}

\textsc{\textbf{Efficient LLM Inference}}. Numerous studies have focused on improving the efficiency of large model inference, including model quantization~\cite{frantar2022gptq,lin2023awq}, model pruning~\cite{gale2019state,sanh2020movement}, and model distillation~\cite{hinton2015distilling}. While these techniques achieve significant speed-ups, they often sacrifice the model's overall effectiveness.
A closely related direction to our work is non-autoregressive decoding, enabling parallel generation of multiple tokens~\citep{gu2017non,wang2019non,sun2019fast,ghazvininejad2019mask,lee2018deterministic,guo2020jointly}. However, these methods typically require extensive retraining of the model and often face challenges in either maintaining model effectiveness or achieving comparable performance without relying on task-specific techniques~\cite{kim2023speculative}.

\textsc{\textbf{Speculative Decoding}}. 
Speculative decoding is initially introduced in~\cite{leviathan2023fast,chen2023accelerating}. More recent works~\cite{sun2023spectr,miao2023specinfer,yang2024multi} extend this concept by allowing the small model to generate multiple draft sequences.
All these methods only maintains a single sequence during generation, making them prone to sub-optimal results. Recently, Andronov et al.~\cite{andronov2024accelerating} proposed a decoding method called ``speculative beam search''. While it retains multiple candidate sequences to handle the chemical synthesis planning task, it does not preserve the same distribution as either beam sampling or multinomial sampling, and their method is fundamentally different from ours.
Another complementary direction to enhance speculative decoding is to improve the effectiveness of the small draft model. A more effective draft model leads to a higher acceptance rate of draft tokens, which in turn accelerates the overall inference process~\cite{kim2023speculative,liu2023online,he2023rest}. EAGLE~\cite{eagle2} and MEDUSA~\cite{cai2024medusa} train additional heads in the target model to generate draft tokens and achieve better acceptance rate.
These works are orthogonal to our work because our algorithm can be directly applied to their draft models.

%% file: sections/conclusion.tex
\section{Conclusion}

This work introduces a novel method that integrates speculative decoding with beam sampling to enhance the efficiency and effectiveness of large language model (LLM) inference. 
Experimental results show that \method{} outperforms beam sampling, achieving a significant speed-up and energy reduction without compromising downstream task performance. 
This work enhances the effectiveness of speculative decoding and opens new avenues for exploration.

\section*{Acknowledgements}

This work was partially supported by NSF grants 2211557, 1937599,  2119643, 2303037, NSF 2312501, SRC JUMP 2.0 PRISM Center, NASA, Okawa Foundation, Amazon Research, Snapchat, and the CDSC industrial partners (https://cdsc.ucla.edu/partners/).

%% file: sections/appendix.tex
\section{Correctness of Speculative Beam Decoding\label{app:correctness}}

Given input beams $\mathcal{I}=\{x_{1:t}^{(1)},\cdots,x_{1:t}^{(W_L)}\}$, and draft beams $\mathcal{S}=\{x_{1:t+1}^{(1)},\cdots,x_{1:t+1}^{(W_S)}\}$ ($W_S> W_L$), let $\mathcal{O}=\{\tilde{x}_{1:t+1}^{(1)},\cdots,\tilde{x}_{1:t+1}^{(W_L)}\}$ denote the output beams obtained by our algorithm (Algorithm \ref{alg:draft_verify}). $\forall i=1,\ldots,W_L$, $\tilde{x}_{1:t+1}^{(i)}\overset{\mathrm{iid}}{\sim} p_{beam}$, where $p_{beam}(x_{1:t+1}^{(i)})= \mathcal{T}\circ(p({x_{1:t+1}^{(i)}})/\mathcal{P})$, $\mathcal{P}=\sum_{x_{1:t+1}\in \mathcal{I}\times V}p(x_{1:t+1})$.

\begin{proof}
    \textbf{Step (1)}. We first prove $\tilde{x}_{1:t+1}^{(1)}\sim p_{beam}$. Let $p_1=p_{beam}$, $p_{k+1}=norm(\max(p_k-q_{beam},0))$ for $k=1,\ldots,W_S-1$.

    Note that we have
    \begin{equation}
     \begin{aligned}
        P(\tilde{x}_{1:t+1}^{(1)}=x_{1:t+1})=&P_{1,1}+P_{2,1}
     \end{aligned}  
    \end{equation}
    where 
    \begin{equation}
        \begin{aligned}
            P_{1,1}=P(&\tilde{x}^{(1)}_{1:t+1}=x_{1:t+1}, x^{(1)}_{1:t+1}\text{ accepted})
        \end{aligned}
    \end{equation}
    \begin{equation}
        \begin{aligned}
            P_{2,1}=P(&\tilde{x}^{(1)}_{1:t+1}=x_{1:t+1}, x^{(1)}_{1:t+1}\text{ rejected})
        \end{aligned}
    \end{equation}
    In addition, for $i>1$ let
    \begin{equation}
        \begin{aligned}
            P_{1,i}=P(&\tilde{x}^{(1)}_{1:t+1}=x_{1:t+1}, x^{(i)}_{1:t+1}\text{ accepted}|\text{previous}\\
            &i-1 \text{ beams all rejected})
        \end{aligned}
    \end{equation}
    \begin{equation}
        \begin{aligned}
            P_{2,i}=P(&\tilde{x}^{(1)}_{1:t+1}=x_{1:t+1}, x^{(i)}_{1:t+1}\text{ rejected}|\text{previous}\\
            &i-1 \text{ beams all rejected})
        \end{aligned}
    \end{equation}
    Notice that
    \begin{equation}
        \begin{aligned}
            P_{1,1}=q_{beam}(x_{1:t+1})\min(1,\frac{p_1(x_{1:t+1})}{q_{beam}(x_{1:t+1})})
        \end{aligned}
    \end{equation}
    where $q_{beam}$ is the beam sampling probability from the small model.
    And we have
    \begin{equation}
        \begin{aligned}
            P_{2,1}
            =&P(x^{(1)}_{1:t+1}\text{ rejected})(P_{1,2}+P_{2,2})\\
            =&(1-\alpha_1)(P_{1,2}+P_{2,2})
        \end{aligned}
    \end{equation}
    where $\alpha_j=\sum_{x_{1:t+1}}q_{beam}(x_{1:t+1})\min(1,\frac{p_j(x_{1:t+1})}{q_{beam}(x_{1:t+1})})$ is the probability that $j$-th beam is accepted given that the first $j-1$ beams are all rejected. 

    Similarly, for $i=2,\ldots,W_S$ we have
    \begin{equation}
        P_{1,i} = q_{beam}(x_{1:t+1})\min(1,\frac{p_i(x_{1:t+1})}{q_{beam}(x_{1:t+1})})
    \end{equation}
    for $i=2,\ldots,W_S-1$ we have
    \begin{equation}
        \begin{aligned}
            P_{2,i}
            =&P(x^{(i)}_{1:t+1}\text{ rejected})(P_{1,i+1}+P_{2,i+1})\\
            =&(1-\alpha_i)(P_{1,i+1}+P_{2,i+1})
        \end{aligned}
    \end{equation}
    for $i=W_S$, we have
    \begin{equation}
        \begin{aligned}
            P_{2,W_S}
            =&(1-\alpha_{W_S})p_{W_S+1}(x_{1:t+1})\\
            =&p_{W_S}(x_{1:t+1})-\min(q_{beam}(x_{1:t+1}), p_{W_S}(x_{1:t+1}))
        \end{aligned}
    \end{equation}
    The last equation is because $p_{W_S+1}(x_{1:t+1})=\frac{p_{W_S}(x_{1:t+1})-\min(q_{beam}(x_{1:t+1}), p_{W_S}(x_{1:t+1}))}{1-\alpha_{W_S}}$.
    
    So
    \begin{equation}
        \begin{aligned}
            &P_{1,W_S}+P_{2,W_S}=p_{W_S}(x_{1:t+1})
        \end{aligned}
    \end{equation}
    With similar steps, we have
    \begin{equation}
        \begin{aligned}
            &P_{1,i}+P_{2,i}=p_{i}(x_{1:t+1}),\quad\forall i=1,\ldots,W_S
        \end{aligned}
    \end{equation}
    Therefore, $P(\tilde{x}_{1:t+1}^{(1)}=x_{1:t+1})=P_{1,1}+P_{2,1}=p_1(x_{1:t+1})=p_{beam}(x_{1:t+1})$.

    \textbf{Step (2)}. Now, assume $\tilde{x}_{1:t+1}^{(i)}\overset{\mathrm{iid}}{\sim} p_{beam}$ for $i=1,\ldots,k$ ($1\le k<W_L$), let us prove $\tilde{x}_{1:t+1}^{(k+1)}\sim p_{beam}$ and it is independent to previous output beams.

    \textit{Case 1: }$\tilde{x}_{1:t+1}^{(k+1)}$ is directly sampled from $p_{beam}$. Then it is obvious that  $\tilde{x}_{1:t+1}^{(k+1)}\sim p_{beam}$ and it is independent to previous output beams.

    \textit{Case 2:} $\tilde{x}_{1:t+1}^{(k+1)}$ is one of the draft beam or it is sampled from residual distribution $p^\prime$. In this case, after output the first $k$ beams, there still are at least one draft beams unverified, and $p^\prime =p_{beam}$. Then with procedure identical to step (1), we can prove $\tilde{x}_{1:t+1}^{(k+1)}\sim p_{beam}$. In addition, since the unverified draft beams are independent to all previous beams, $\tilde{x}_{1:t+1}^{(k+1)}$ is independent to previous output beams.

    With step (1) and (2), we prove the correctness of our algorithm.
\end{proof}

\section{Experiment Configuration}

The experiments are conducted on a server with one NVIDIA-L40 GPU. We use top-$k$ and top-$p$ warping in the experiments with $k=10$, $p=0.8$.

We use the command "\texttt{nvidia-smi --query-gpu=power.draw --format=csv}" to get GPU power every second, and sum them up as the total energy consumption. 
We use average energy consumption per token to measure energy efficiency. We follow the three principles proposed in~\cite{yang2023part} to minimize the measurement error using \texttt{nvidia-smi}.

\section{Supplementary Experiments}

Here we demonstrate some supplementary experiments about our proposed method.

\subsection{OPT evaluated on Spider}

Figure \ref{fig:opt_spider} shows that \method{} achieves similar perplexity (i.e., likelihood) of output sequences with beam sampling while being much more efficient.

\begin{figure}[h!]
    \centering
    \begin{subfigure}[b]{0.22\textwidth}
        \centering
        
        \includegraphics[width=\textwidth]{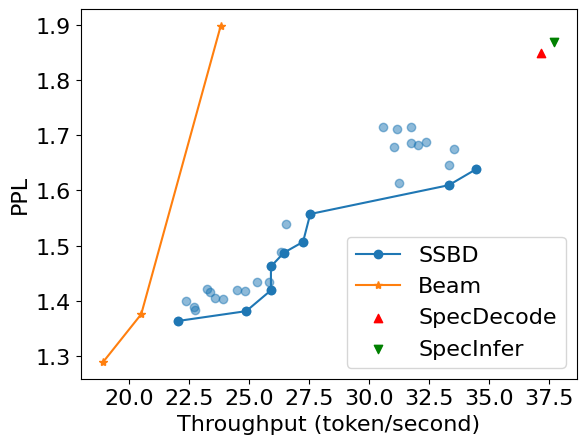}
        \caption{Speed vs Perplexity}
        \label{fig:opt_spider_speed}
    \end{subfigure}
    \hfill
    \begin{subfigure}[b]{0.22\textwidth}
        \centering
        \includegraphics[width=\textwidth]{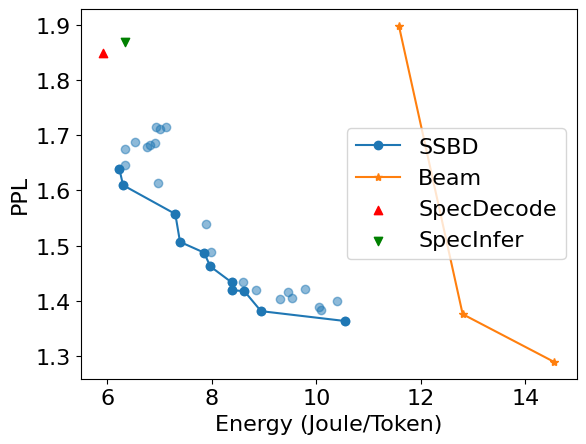}
        \caption{Energy vs Perplexity}
        \label{fig:opt_spider_energy}
    \end{subfigure}
    \caption{Evaluation on Spider dataset with OPT models. Perplexity (PPL) is \textbf{lower} the better. The blue points represent performances of \method{} under different parameter settings $(\gamma,W_S,t)$. The blue and yellow lines mark the Pareto fronts of \method{} and beam sampling.}
    \label{fig:opt_spider}
\end{figure}

\subsection{Average Accepted Width}

In beam sampling, the beam width directly controls the trade-off between efficiency and effectiveness. We believe the average accepted number of beams or average accepted width (denoted as $\bar{W}$) has a similar effect to \method{}. To confirm this, we plot the relationship between $\bar{W}$ and downstream metric (EA Score), speed, and energy consumption in Figure \ref{fig:W}, respectively. We can see that in general, a larger $\bar{W}$ corresponds to a higher EA score, slower speed, and higher energy consumption, which fits our intuition.

\begin{figure}[h!]
    \centering
    \begin{subfigure}[b]{0.22\textwidth}
        \centering
        \includegraphics[width=\textwidth]{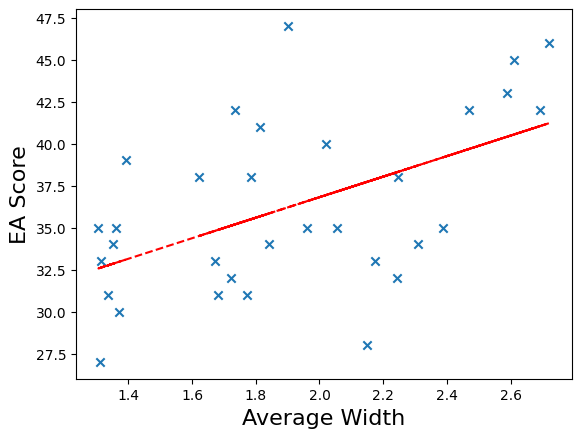}
        \caption{$\bar{W}$ vs EA Score}
        \label{fig:w_score}
    \end{subfigure}
    \hfill
    \begin{subfigure}[b]{0.22\textwidth}
        \centering
        \includegraphics[width=\textwidth]{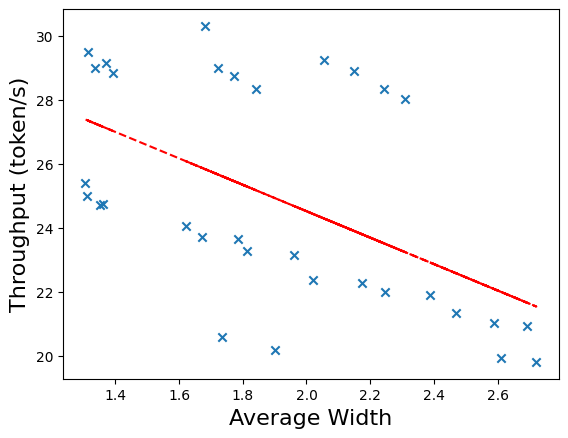}
        \caption{$\bar{W}$ vs Speed}
        \label{fig:w_speed}
    \end{subfigure}
    
    \begin{subfigure}[b]{0.22\textwidth}
        \centering
        \includegraphics[width=\textwidth]{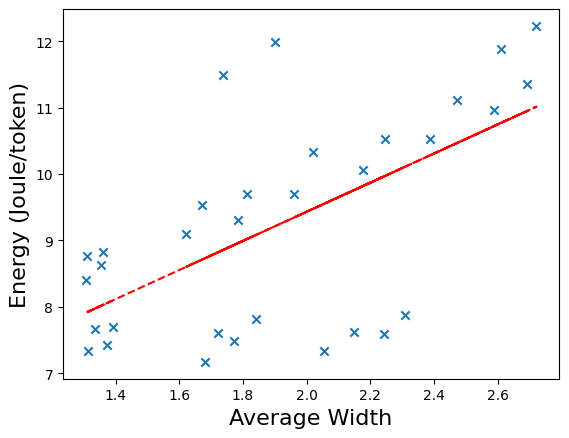}
        \caption{$\bar{W}$ vs Energy}
        \label{fig:w_energy}
    \end{subfigure}
    \caption{Relationship between average accepted width ($\bar{W}$) and downstream effectiveness and efficiency. (Llama-2 models, Spider dataset)}
    \label{fig:W}
\end{figure}

\begin{figure}[h!]
    \centering
    \begin{subfigure}[b]{0.22\textwidth}
        \centering
        \includegraphics[width=\textwidth]{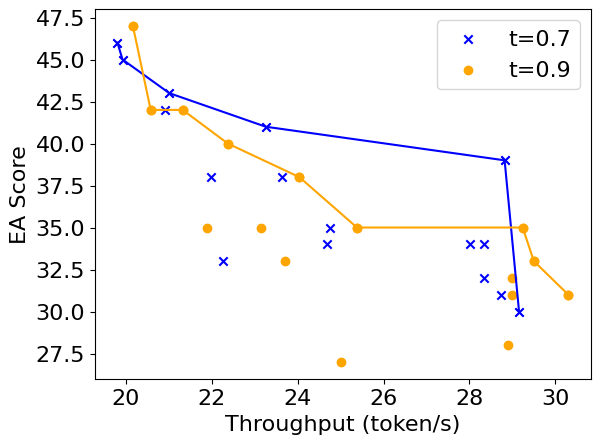}
        \caption{Speed vs EA Score}
        \label{fig:t_speed}
    \end{subfigure}
    \hfill
    \begin{subfigure}[b]{0.22\textwidth}
        \centering
        \includegraphics[width=\textwidth]{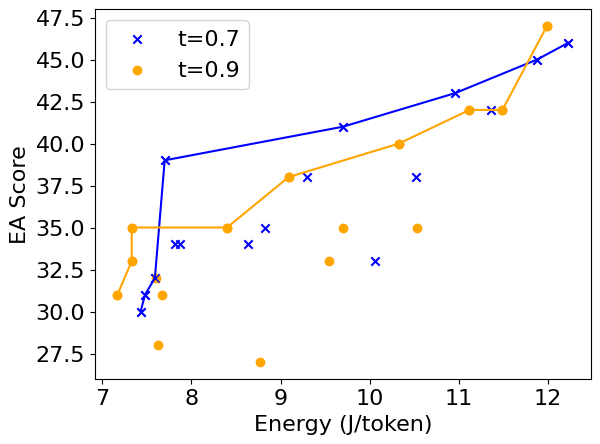}
        \caption{Energy vs EA Score}
        \label{fig:t_energy}
    \end{subfigure}
    \caption{Performance of \method{} with dynamic width threshold $t\in\{0.7,0.9\}$. (Llama-2, Spider dataset)}
    \label{fig:t}
\end{figure}

\begin{figure}[h!]
    \centering
    \begin{subfigure}[b]{0.22\textwidth}
        \centering
        \includegraphics[width=\textwidth]{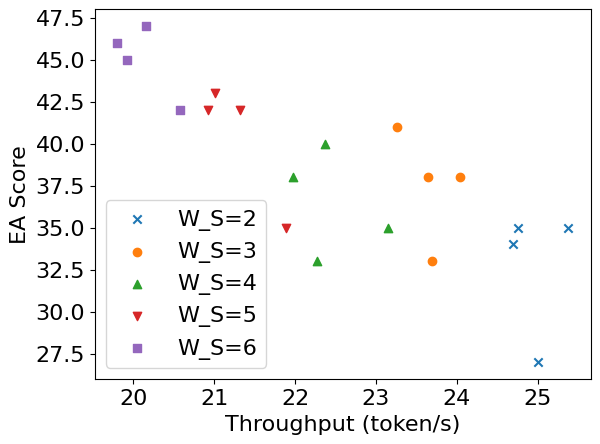}
        \caption{Speed vs EA Score}
        \label{fig:W_S_speed}
    \end{subfigure}
    \hfill
    \begin{subfigure}[b]{0.22\textwidth}
        \centering
        \includegraphics[width=\textwidth]{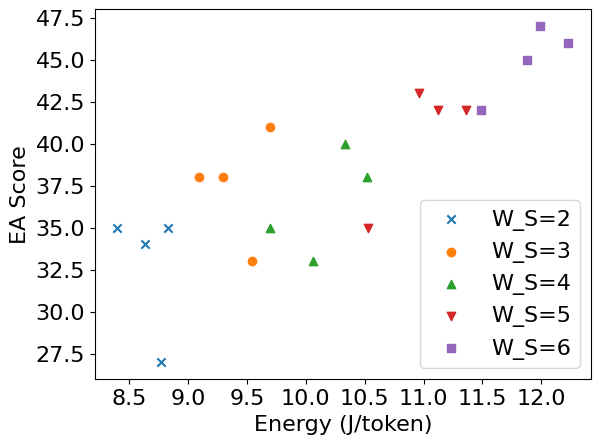}
        \caption{Energy vs EA Score}
        \label{fig:W_S_energy}
    \end{subfigure}
    \caption{Performance of \method{} with different draft width $W_S\in\{2,3,4,5,6\}$. (Llama-2, Spider dataset)}
    \label{fig:W_S}
\end{figure}

\subsection{Dynamic Width Threshold}

Now, we evaluate how the dynamic width threshold $t$ (see Eq \ref{eq:WL}) affects our method. We show the performance of $t=0.7$ and $t=0.9$ in Figure \ref{fig:t}. First, we see that $t=0.7$ tends to achieve a higher effectiveness. It is because a smaller $t$ leads to a larger $W_L^{(l)}$, which would increase the effectiveness. However, setting $t=0.9$ could achieve a higher speed and energy efficiency that setting $t=0.7$ cannot achieve.

\subsection{Number of Draft Beams}

Next we evaluate how the number of draft beams (i.e., $W_S$) affects the performance of \method{}. We show the performance under different $W_S$ in Figure \ref{fig:W_S}. we can see that a higher $W_S$ in general leads to lower efficiency but higher effectiveness. It is because a higher $W_S$ will increase the cost of running small model, large model, and verification. But it also helps increase the number of beams accepted at each layer.

%% file: main.bbl
\begin{thebibliography}{37}
\providecommand{\natexlab}[1]{#1}

\bibitem[{Achiam et~al.(2023)Achiam, Adler, Agarwal, Ahmad, Akkaya, Aleman, Almeida, Altenschmidt, Altman, Anadkat et~al.}]{achiam2023gpt}
Achiam, J.; Adler, S.; Agarwal, S.; Ahmad, L.; Akkaya, I.; Aleman, F.~L.; Almeida, D.; Altenschmidt, J.; Altman, S.; Anadkat, S.; et~al. 2023.
\newblock Gpt-4 technical report.
\newblock \emph{arXiv preprint arXiv:2303.08774}.

\bibitem[{AI@Meta(2024)}]{llama3modelcard}
AI@Meta. 2024.
\newblock Llama 3 Model Card.

\bibitem[{Allen and Ge(2016)}]{allen2016characterizing}
Allen, T.; and Ge, R. 2016.
\newblock Characterizing power and performance of gpu memory access.
\newblock In \emph{2016 4th International Workshop on Energy Efficient Supercomputing (E2SC)}, 46--53. IEEE.

\bibitem[{Andronov et~al.(2024)Andronov, Andronova, Wand, Schmidhuber, and Clevert}]{andronov2024accelerating}
Andronov, M.; Andronova, N.; Wand, M.; Schmidhuber, J.; and Clevert, D.-A. 2024.
\newblock Accelerating the inference of string generation-based chemical reaction models for industrial applications.
\newblock \emph{arXiv preprint arXiv:2407.09685}.

\bibitem[{Anil et~al.(2023)Anil, Dai, Firat, Johnson, Lepikhin, Passos, Shakeri, Taropa, Bailey, Chen et~al.}]{anil2023palm}
Anil, R.; Dai, A.~M.; Firat, O.; Johnson, M.; Lepikhin, D.; Passos, A.; Shakeri, S.; Taropa, E.; Bailey, P.; Chen, Z.; et~al. 2023.
\newblock Palm 2 technical report.
\newblock \emph{arXiv preprint arXiv:2305.10403}.

\bibitem[{Cai et~al.(2024)Cai, Li, Geng, Peng, Lee, Chen, and Dao}]{cai2024medusa}
Cai, T.; Li, Y.; Geng, Z.; Peng, H.; Lee, J.~D.; Chen, D.; and Dao, T. 2024.
\newblock Medusa: Simple llm inference acceleration framework with multiple decoding heads.
\newblock \emph{arXiv preprint arXiv:2401.10774}.

\bibitem[{Chen et~al.(2023)Chen, Borgeaud, Irving, Lespiau, Sifre, and Jumper}]{chen2023accelerating}
Chen, C.; Borgeaud, S.; Irving, G.; Lespiau, J.-B.; Sifre, L.; and Jumper, J. 2023.
\newblock Accelerating large language model decoding with speculative sampling.
\newblock \emph{arXiv preprint arXiv:2302.01318}.

\bibitem[{Chen et~al.(2011)Chen, Li, Zhang, Peng, and Peir}]{chen2011tree}
Chen, J.; Li, B.; Zhang, Y.; Peng, L.; and Peir, J.-k. 2011.
\newblock Tree structured analysis on GPU power study.
\newblock In \emph{2011 IEEE 29th International Conference on Computer Design (ICCD)}, 57--64. IEEE.

\bibitem[{Frantar et~al.(2022)Frantar, Ashkboos, Hoefler, and Alistarh}]{frantar2022gptq}
Frantar, E.; Ashkboos, S.; Hoefler, T.; and Alistarh, D. 2022.
\newblock Gptq: Accurate post-training quantization for generative pre-trained transformers.
\newblock \emph{arXiv preprint arXiv:2210.17323}.

\bibitem[{Gale, Elsen, and Hooker(2019)}]{gale2019state}
Gale, T.; Elsen, E.; and Hooker, S. 2019.
\newblock The state of sparsity in deep neural networks.
\newblock \emph{arXiv preprint arXiv:1902.09574}.

\bibitem[{Ghazvininejad et~al.(2019)Ghazvininejad, Levy, Liu, and Zettlemoyer}]{ghazvininejad2019mask}
Ghazvininejad, M.; Levy, O.; Liu, Y.; and Zettlemoyer, L. 2019.
\newblock Mask-predict: Parallel decoding of conditional masked language models.
\newblock \emph{arXiv preprint arXiv:1904.09324}.

\bibitem[{Gu et~al.(2017)Gu, Bradbury, Xiong, Li, and Socher}]{gu2017non}
Gu, J.; Bradbury, J.; Xiong, C.; Li, V.~O.; and Socher, R. 2017.
\newblock Non-autoregressive neural machine translation.
\newblock \emph{arXiv preprint arXiv:1711.02281}.

\bibitem[{Guo, Xu, and Chen(2020)}]{guo2020jointly}
Guo, J.; Xu, L.; and Chen, E. 2020.
\newblock Jointly masked sequence-to-sequence model for non-autoregressive neural machine translation.
\newblock In \emph{Proceedings of the 58th Annual Meeting of the Association for Computational Linguistics}, 376--385.

\bibitem[{He et~al.(2023)He, Zhong, Cai, Lee, and He}]{he2023rest}
He, Z.; Zhong, Z.; Cai, T.; Lee, J.~D.; and He, D. 2023.
\newblock Rest: Retrieval-based speculative decoding.
\newblock \emph{arXiv preprint arXiv:2311.08252}.

\bibitem[{Hinton, Vinyals, and Dean(2015)}]{hinton2015distilling}
Hinton, G.; Vinyals, O.; and Dean, J. 2015.
\newblock Distilling the knowledge in a neural network.
\newblock \emph{arXiv preprint arXiv:1503.02531}.

\bibitem[{Kang et~al.(2024)Kang, Zhang, Kundu, Jeong, Liu, Krishna, and Zhao}]{kang2024gear}
Kang, H.; Zhang, Q.; Kundu, S.; Jeong, G.; Liu, Z.; Krishna, T.; and Zhao, T. 2024.
\newblock Gear: An efficient kv cache compression recipefor near-lossless generative inference of llm.
\newblock \emph{arXiv preprint arXiv:2403.05527}.

\bibitem[{Kim et~al.(2023)Kim, Mangalam, Moon, Malik, Mahoney, Gholami, and Keutzer}]{kim2023speculative}
Kim, S.; Mangalam, K.; Moon, S.; Malik, J.; Mahoney, M.~W.; Gholami, A.; and Keutzer, K. 2023.
\newblock Speculative decoding with big little decoder.
\newblock In \emph{Thirty-seventh Conference on Neural Information Processing Systems}.

\bibitem[{Lee, Mansimov, and Cho(2018)}]{lee2018deterministic}
Lee, J.; Mansimov, E.; and Cho, K. 2018.
\newblock Deterministic non-autoregressive neural sequence modeling by iterative refinement.
\newblock \emph{arXiv preprint arXiv:1802.06901}.

\bibitem[{Leviathan, Kalman, and Matias(2023)}]{leviathan2023fast}
Leviathan, Y.; Kalman, M.; and Matias, Y. 2023.
\newblock Fast inference from transformers via speculative decoding.
\newblock In \emph{International Conference on Machine Learning}, 19274--19286. PMLR.

\bibitem[{Li et~al.(2024)Li, Wei, Zhang, and Zhang}]{eagle2}
Li, Y.; Wei, F.; Zhang, C.; and Zhang, H. 2024.
\newblock EAGLE-2: Faster Inference of Language Models with Dynamic Draft Trees.
\newblock \emph{arXiv preprint arXiv:2406.16858}.

\bibitem[{Lin et~al.(2023)Lin, Tang, Tang, Yang, Dang, and Han}]{lin2023awq}
Lin, J.; Tang, J.; Tang, H.; Yang, S.; Dang, X.; and Han, S. 2023.
\newblock AWQ: Activation-aware Weight Quantization for LLM Compression and Acceleration.
\newblock \emph{arXiv preprint arXiv:2306.00978}.

\bibitem[{Liu et~al.(2023)Liu, Hu, Bailis, Stoica, Deng, Cheung, and Zhang}]{liu2023online}
Liu, X.; Hu, L.; Bailis, P.; Stoica, I.; Deng, Z.; Cheung, A.; and Zhang, H. 2023.
\newblock Online speculative decoding.
\newblock \emph{arXiv preprint arXiv:2310.07177}.

\bibitem[{Miao et~al.(2023)Miao, Oliaro, Zhang, Cheng, Wang, Wong, Chen, Arfeen, Abhyankar, and Jia}]{miao2023specinfer}
Miao, X.; Oliaro, G.; Zhang, Z.; Cheng, X.; Wang, Z.; Wong, R. Y.~Y.; Chen, Z.; Arfeen, D.; Abhyankar, R.; and Jia, Z. 2023.
\newblock Specinfer: Accelerating generative llm serving with speculative inference and token tree verification.
\newblock \emph{arXiv preprint arXiv:2305.09781}, 1(2): 4.

\bibitem[{Qin et~al.(2024)Qin, He, Prakriya, Cong, and Sun}]{qin2024dynamic}
Qin, Z.; He, Z.; Prakriya, N.; Cong, J.; and Sun, Y. 2024.
\newblock Dynamic-Width Speculative Beam Decoding for Efficient LLM Inference.
\newblock \emph{arXiv preprint arXiv:2409.16560}.

\bibitem[{Rajpurkar, Jia, and Liang(2018)}]{squad}
Rajpurkar, P.; Jia, R.; and Liang, P. 2018.
\newblock Know what you don't know: Unanswerable questions for SQuAD.
\newblock \emph{arXiv preprint arXiv:1806.03822}.

\bibitem[{Sanh, Wolf, and Rush(2020)}]{sanh2020movement}
Sanh, V.; Wolf, T.; and Rush, A. 2020.
\newblock Movement pruning: Adaptive sparsity by fine-tuning.
\newblock \emph{Advances in Neural Information Processing Systems}, 33: 20378--20389.

\bibitem[{Shi et~al.(2024)Shi, Yang, Cai, Zhang, Wang, Yang, and Lam}]{shi2024thorough}
Shi, C.; Yang, H.; Cai, D.; Zhang, Z.; Wang, Y.; Yang, Y.; and Lam, W. 2024.
\newblock A thorough examination of decoding methods in the era of llms.
\newblock \emph{arXiv preprint arXiv:2402.06925}.

\bibitem[{Sun et~al.(2019)Sun, Li, Wang, He, Lin, and Deng}]{sun2019fast}
Sun, Z.; Li, Z.; Wang, H.; He, D.; Lin, Z.; and Deng, Z. 2019.
\newblock Fast structured decoding for sequence models.
\newblock \emph{Advances in Neural Information Processing Systems}, 32.

\bibitem[{Sun et~al.(2023)Sun, Suresh, Ro, Beirami, Jain, and Yu}]{sun2023spectr}
Sun, Z.; Suresh, A.~T.; Ro, J.~H.; Beirami, A.; Jain, H.; and Yu, F. 2023.
\newblock Spectr: Fast speculative decoding via optimal transport.
\newblock \emph{arXiv preprint arXiv:2310.15141}.

\bibitem[{Touvron et~al.(2023)Touvron, Lavril, Izacard, Martinet, Lachaux, Lacroix, Rozi{\`e}re, Goyal, Hambro, Azhar et~al.}]{touvron2023llama}
Touvron, H.; Lavril, T.; Izacard, G.; Martinet, X.; Lachaux, M.-A.; Lacroix, T.; Rozi{\`e}re, B.; Goyal, N.; Hambro, E.; Azhar, F.; et~al. 2023.
\newblock Llama: Open and efficient foundation language models.
\newblock \emph{arXiv preprint arXiv:2302.13971}.

\bibitem[{Vaswani et~al.(2017)Vaswani, Shazeer, Parmar, Uszkoreit, Jones, Gomez, Kaiser, and Polosukhin}]{vaswani2017attention}
Vaswani, A.; Shazeer, N.; Parmar, N.; Uszkoreit, J.; Jones, L.; Gomez, A.~N.; Kaiser, {\L}.; and Polosukhin, I. 2017.
\newblock Attention is all you need.
\newblock \emph{Advances in neural information processing systems}, 30.

\bibitem[{Wang et~al.(2019)Wang, Tian, He, Qin, Zhai, and Liu}]{wang2019non}
Wang, Y.; Tian, F.; He, D.; Qin, T.; Zhai, C.; and Liu, T.-Y. 2019.
\newblock Non-autoregressive machine translation with auxiliary regularization.
\newblock In \emph{Proceedings of the AAAI conference on artificial intelligence}, volume~33, 5377--5384.

\bibitem[{Yang et~al.(2024)Yang, Huang, Dai, and Chen}]{yang2024multi}
Yang, S.; Huang, S.; Dai, X.; and Chen, J. 2024.
\newblock Multi-candidate speculative decoding.
\newblock \emph{arXiv preprint arXiv:2401.06706}.

\bibitem[{Yang, Adamek, and Armour(2023)}]{yang2023part}
Yang, Z.; Adamek, K.; and Armour, W. 2023.
\newblock Part-time Power Measurements: nvidia-smi's Lack of Attention.
\newblock \emph{arXiv preprint arXiv:2312.02741}.

\bibitem[{Yu et~al.(2018)Yu, Zhang, Yang, Yasunaga, Wang, Li, Ma, Li, Yao, Roman et~al.}]{yu2018spider}
Yu, T.; Zhang, R.; Yang, K.; Yasunaga, M.; Wang, D.; Li, Z.; Ma, J.; Li, I.; Yao, Q.; Roman, S.; et~al. 2018.
\newblock Spider: A large-scale human-labeled dataset for complex and cross-domain semantic parsing and text-to-sql task.
\newblock \emph{arXiv preprint arXiv:1809.08887}.

\bibitem[{Zhang et~al.(2022)Zhang, Roller, Goyal, Artetxe, Chen, Chen, Dewan, Diab, Li, Lin et~al.}]{zhang2022opt}
Zhang, S.; Roller, S.; Goyal, N.; Artetxe, M.; Chen, M.; Chen, S.; Dewan, C.; Diab, M.; Li, X.; Lin, X.~V.; et~al. 2022.
\newblock Opt: Open pre-trained transformer language models.
\newblock \emph{arXiv preprint arXiv:2205.01068}.

\bibitem[{Zheng et~al.(2023)Zheng, Chiang, Sheng, Zhuang, Wu, Zhuang, Lin, Li, Li, Xing et~al.}]{zheng2023judging}
Zheng, L.; Chiang, W.-L.; Sheng, Y.; Zhuang, S.; Wu, Z.; Zhuang, Y.; Lin, Z.; Li, Z.; Li, D.; Xing, E.; et~al. 2023.
\newblock Judging llm-as-a-judge with mt-bench and chatbot arena.
\newblock \emph{Advances in Neural Information Processing Systems}, 36: 46595--46623.

\end{thebibliography}
